\documentclass[twocolumn,10pt]{article}
\usepackage{blindtext} 
\usepackage{authblk}
\usepackage[sc]{mathpazo} 
\usepackage[T1]{fontenc} 
\usepackage{microtype} 

\usepackage[utf8]{inputenc}

\usepackage[english]{babel} 

\usepackage[a4paper,left=1.5cm,right=1.5cm,top=3cm,bottom=2cm]{geometry}
\usepackage[hang, small,labelfont=bf,up,textfont=it,up]{caption} 
\usepackage{booktabs} 

\usepackage{lettrine} 

\usepackage{enumitem} 
\setlist[itemize]{noitemsep} 

\usepackage{abstract} 

\usepackage{titlesec} 
\renewcommand\thesection{\Roman{section}} 
\renewcommand\thesubsection{\roman{subsection}} 
\titleformat{\section}[block]{\large\scshape\centering}{\thesection.}{1em}{} 
\titleformat{\subsection}[block]{\large}{\thesubsection.}{1em}{} 

\usepackage{fancyhdr} 
\usepackage{titling} 

\usepackage{hyperref} 

\usepackage{hyperref}
\usepackage{url}
\usepackage{graphicx}
\usepackage{tikz}
\usepackage{tabu}
\def\layersep{2.5cm}
\usetikzlibrary{arrows,shapes,chains}
\usepackage{amsthm}
\usepackage{amsmath}

\newtheorem*{thm}{Theorem}

\pretitle{\begin{center}\Huge\bfseries} 
\posttitle{\end{center}} 
\title{Feedforward Neural Network for Time Series Anomaly Detection} 

\author{%
\textsc{ZHANG Rong} \\
\normalsize Tencent Company \\ 
\normalsize \href{zr9558@gmail.com}{zr9558@gmail.com} 
\\
\textsc{NIE Xin}\\
\normalsize Tencent Company \\ 
\normalsize \href{michaelnie@tencent.com}{michaelnie@tencent.com} 
\and
\textsc{DONG Shandong}\\
\normalsize National University of Singapore \\ 
\normalsize \href{E0009088@u.nus.edu}{E0009088@u.nus.edu} 
\\
\textsc{XIAO Shiguang}\\ 
\normalsize Tencent Company \\ 
\normalsize \href{philipxiao@tencent.com}{philipxiao@tencent.com} 
}

\date{\today} 
\renewcommand{\maketitlehookd}{%
\begin{abstract}
Time series anomaly detection is usually formulated as finding outlier data points relative to some usual data, which is also an important problem in industry and academia. To ensure systems working stably, internet companies, banks and other companies need to monitor time series, which is called KPI (Key Performance Indicators), such as CPU used, number of orders, number of online users and so on. However, millions of time series have several shapes (e.g. seasonal KPIs, KPIs of timed tasks and KPIs of CPU used), so that it is very difficult to use a simple statistical model to detect anomaly for all kinds of time series. Although some anomaly detectors have developed many years and some supervised models are also available in this field, we find many methods have their own disadvantages. In this paper, we present our system, which is based on deep  feedforward neural network and detect anomaly points of time series. The main difference between our system and other systems based on supervised models is that we do not need feature engineering of time series to train deep feedforward  neural network in our system, which is essentially an end-to-end system.
\end{abstract}
}

\begin{document}

\maketitle

\section{Introduction}
To ensure systems working stably and efficiently, internet companies need to monitor huge time series every minute, whose names are KPIs (Key Performance Indicators). For example, in the industry and academia, KPIs contain several kinds of time series, including CPU, online page views, online users of some application, the number of failures and successes of logging some website. In our opinion, different KPIs have different shapes and trends, so it is difficult to use a simple statistical model to detect all anomaly points of KPIs. Before using machine learning models, we wrote rules and used 3-sigma method to detect anomaly points of time series. However, rules become more and more complicated and it is impossible for us to check all rules at regular time intervals in order to guarantee precision and recall of the whole system. Therefore, we try our best to build a new system which is based on human experience and machine learning theory, in order to increase the precision and recall of the our system for anomaly detection of time series.

Anomaly detection has been an active research area in the fields of machine learning and statistics. Statistical methods, control chart theory \cite{StatisticalQualityControl}, ARIMA and seasonal ARIMA models \cite{ARIMA3},\cite{ARIMA1},\cite{ARIMA2}, Holt-Winters model \cite{HoltWinters} are proposed for time series anomaly detection. Beside statistical models, in machine learning theory, there are also a lot of methods to detect anomaly points of time series, such as supervised and unsupervised models. Most existing anomaly detection approaches, including classification-based methods \cite{Aggarwal}, isolation forest \cite{IsolationForest}, one-class SVM \cite{OneClassSVM}, clustering-based methods, construct normal pattern from samples, then identify anomaly points as those which do not satisfy the normal pattern. In the field of time series anomaly detection, some scholars provided supervised models bases on feature engineering \cite{Opprentice} and unsupervised models \cite{Donut} to detect anomaly points of KPIs.

As discussed before, to ensure high precision and recall of the system, we must provide a stable and efficient algorithm. This paper proposes a different approach that detects anomalies by neural networks, without relying on any feature engineering or time series detectors. The contributions of the paper can be summarized as follows:

\begin{itemize}
\item Deep feedforward neural network is an end-to-end model, which can be trained from normalized raw datas to corresponding labels. The main technique in our system is that there is no need of feature engineering of time series.
\item The offline precision and recall of deep feedforward neural network are higher than XGBoost with a lot of feature engineering from human experience.
\item The output of hidden layers of trained feedforward neural networks can be taken as the features of time series.
\end{itemize}

\section{Background and Problems}
In this section, we focus on the statement of the academic problem on time series anomaly detection and describe its background in the industry.

\subsection{Time Series Anomalies}
In the industry, time series are collected from network logs, computers and applications. In the industry, we can collect time series on response time, successful rate, failure count, the number of online users and so on. 

Anomaly points of time series are some points with unexpected patterns (e.g. suddenly increasing and decreasing, trends changes, level shifts and exceeding the maximal value in the history). These anomaly points in time series means some network services do not work well enough, some users can not login on applications or open some websites. Therefore, time series anomaly detection is a crucial work in our daily work, and we must provide a stable and robust model to detect anomaly points for a great quantity of time series in real time.

\subsection{Problem and Goal}
In machine learning theory, the fundamental performance of supervised models contains two important indices, which are called recall and precision. In anomaly detection of  time series, if we use negative label to denote the anomaly case and positive label to denote the normal case, then the recall and precision of supervised models can be defined as
\begin{eqnarray*}
& &\mbox{Recall} \\
&=& \frac{\mbox{the number of true anomalous points detected}}{\mbox{the number of true anomalous points}}\\
&=& \frac{TN}{TN+FP}, \\
& &\mbox{Precision} \\&=& \frac{\mbox{the number of true anomalous points detected}}{\mbox{the number of anomalous points detected}}\\
&=& \frac{TN}{TN+FN},
\label{RecallAndPrecision}
\end{eqnarray*}
respectively. The notations TP, FN, FP and TN mean true positive, false negative, false positive, true negative, whose details are in table $\ref{ConfusionMatrix}$. Besides precision and recall, F1-score is comprehensive metric to measure the performance of supervised models, which is defined as
\begin{eqnarray*}
\mbox{F1-Score} = \frac{2\cdot\mbox{precision}\cdot\mbox{recall}}{\mbox{precision} + \mbox{recall}}.
\label{F1-Score}
\end{eqnarray*}
In time series anomaly detection, we wish to build a robust supervised model such that recall and precision of it as high as possible. Especially, we wish to use some supervised model such as feedforward neural network to reduce feature engineering of time series, even avoid feature engineering.

\begin{table}
\begin{center}
\begin{tabular}{ | m{8em} | m{2cm}| m{2cm} | }
\hline
 & \multicolumn{2}{c|}{Prediction Results} \\
\cline{2-3}
True Cases& Positive & Negative \\
\hline
Positive & TP & FN \\
\hline
Negative & FP & TN \\
\hline
\end{tabular}
\caption{Confusion Matrix}
\label{ConfusionMatrix}
\end{center}
\end{table}

\subsection{Previous Work}
In order to detect anomalies of time series, there exist a lot of useful methods in statistics. For example, we can use control chart theory, such as basic control chart, moving average control chart, exponentially weighted moving average control chart, which are written in the book clearly \cite{StatisticalQualityControl}. In the theory of time series, ARIMA models is an excellent model for anomaly detection of time series, there are a lot of papers on this method \cite{TimeSeriesAnalysis}\cite{ARIMA1}\cite{ARIMA2}.Weierstrass approximation theorem states, every continuous function defined on a closed interval $[a, b]$ can be uniformly approximated as closely as desired by a polynomial function \cite{Rudin1}\cite{Rudin2}. Therefore, we can use a polynomial to approximate some time series, and get its anomaly points. In Tsinghua and Baidu's work \cite{Opprentice}, they provide a time series anomaly detection system, which is called 'Opprentice'. In the paper, they use $64$ basic detectors to calculate $133$ features for time series and get labels from experiences of people. The tool for time series anomaly labeling in Baidu can be download in github \cite{BaiduCurve}. Recently, Microsoft also opened the other anomaly detection labelling tool in Github \cite{MicrosoftCurve}. In \cite{Opprentice}, they use labels and features to train a supervised models (random forest), finally they got a trained model to predict anomaly points for incoming time series. Besides supervised models, unsupervised models, such as variational auto encoder, can be used to detect anomaly points for seasonal time series, which is in Tsinghua and Alibaba's work \cite{Donut}. Recurrent neural networks and LSTMs can be used for anomaly detection in time series \cite{LSTMForAnomalyDetection}, but one LSTM model can only detect anomaly points for one of time series.

\subsection{Challenges in Time Series Anomaly Detection}

\begin{itemize}
\item \textbf{Big Amount of Time Series.} In the industry, we usually collect time series every minutes, that means there are 1440 points in one time series every day. Beside it, the number of time series exceeds one million, therefore every minutes we must detect anomaly points in this timestamp for millions of time series.

\item \textbf{Class Imbalance Problem.} In machine learning theory, most classification datasets do not have exactly equal number of instances in each class. In reality, the vase majority of the points in time series are in the normal class and a very small minority are in the abnormal class, i.e. the number of normal cases are much larger than the number of anomaly cases. Before we train supervised models, we must use tactics to combat imbalanced training data.

\item \textbf{Incomplete Normal and Anomaly Cases.} In the industry and academia, since there are too many different shapes of time series and the anomaly cases are much less than normal cases, it is difficult for us to label all timestamps of time series. Moreover, how to label normal and abnormal points depends on human's experience. 

\item \textbf{Feature Engineering is Complex.} If we use supervised models to train labelled datasets, then we usually try to construct enough features from original time series. However, feature engineering depends on human experience, and different engineers will construct different features for the same time series. It is possible for engineers to construct redundant, complex and even contradictory features. Too many features will result in the efficiency of anomaly detection in real time, since in one minute we must calculate all features and predict the status for millions of time series from the trained supervised model.

\end{itemize}

\section{System Overview}

\subsection{Core Ideas}
In this section, we pay attention to the whole system design for time series anomaly detection in figure $\ref{MetisSystemOverview}$. In the whole system, the main difference between our system and other systems whose are based on supervised models is feature engineering of time series. For instance, in Opprentice system\cite{Opprentice}, they use $64$ detectors to calculate $133$ features. However, in the system we do not construct any features of time series by human experience. In the main theorem of the paper, we will prove that there exists a deep feedforward neural network to calculate all features of time series in table $\ref{FeatureEngineeringofTimeSeries}$. Therefore, we try to construct an end-to-end model, which means the input is the original datasets, the output is the corresponding labels. Then we try to make use of deep neural network to train raw datasets and their labels and get a trained neural network. Then the outputs of hidden layers can be taken as the features of time series, which is called "Time Series To Vectors".

\begin{figure}
\begin{center}
\resizebox{8cm}{2.5cm}{
\begin{tikzpicture}[
roundnode/.style={circle, draw=green!60, fill=green!5, very thick, minimum size=4mm},
squarednode/.style={rectangle, draw=green!60, fill=green!5, very thick, minimum size=3mm},
]
\scalebox{1}{
\node[squarednode]      (history)                              {Historical Time Series};
\node[squarednode]      (label)    [right=of history]       {Labeling Tool};
\node[squarednode]      (DNNTraining)    [right=of label]       {Training Deep Feedforward Neural Network Models};
\node[squarednode]        (incoming)       [below=of history] {Incoming Time Series};
\node[squarednode] (DNNTrained) [below=of DNNTraining] {Trained Deep Feedforward Neural Network};
\node[squarednode] (prediction) [below=of DNNTrained] {Prediction: Anomaly or Normal?};

\draw[->] (history.east) -- (label.west);
\draw[->] (label.east) -- (DNNTraining.west);
\draw[->] (DNNTraining.south) -- (DNNTrained.north);
\draw[->] (incoming.east) -- (DNNTrained.west);
\draw[->] (DNNTrained.south) -- (prediction.north);
}
\end{tikzpicture}
}
\caption{The Whole System Overview}\label{MetisSystemOverview}
\end{center}
\end{figure}
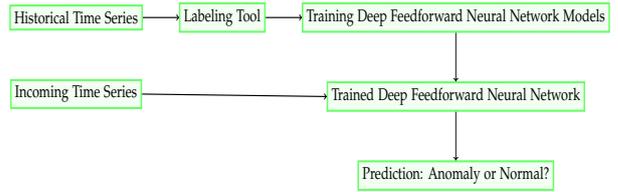

\subsection{Details of Deep Feedforward Neural Networks}
Before explaining the details of the whole system, we must write the main theorem of the paper in the section. The table $\ref{FeatureEngineeringofTimeSeries}$ contains some of features in \cite{Opprentice} and other classical features of time series. 

\begin{table}
\begin{center}
\begin{tabular}{ |  m{12em}| m{10em} | }
\hline
Detectors and Some Features of Time Series& Parameters \\
\hline
simple threshold & none\\
\hline
max, min, average  & none \\
\hline
difference, integration  &  none \\
\hline
 absolute sum of changes, mean change, mean second derivative central & none\\
\hline
count above mean, count below mean & none \\
\hline
historical change & window size = 1, 7 days \\
\hline
Simple Moving Average & window size = 10, 20, 30, 40, 50 minutes \\
\hline
Weighted Moving Average & window size = 10, 20, 30, 40, 50 minutes\\
\hline
Exponentially Weighted Moving Average & $\alpha$ = 0.2, 0.4, 0.6, 0.8 \\
\hline
\end{tabular}
\caption{Feature Engineering of Time Series}\label{FeatureEngineeringofTimeSeries}
\end{center}
\end{table}

\begin{thm}
Suppose $n\geq 1$ is a positive integer, there exists a deep feedforward neural network $D$ such that for any real time series $\mathbf{X}_{n}=[X_{1},\cdots, X_{n}]$, the input and output of $D$ are  $\mathbf{X}_{n}$ and the features of $\mathbf{X}_{n}$ in table $\ref{FeatureEngineeringofTimeSeries}$, respectively.
\end{thm}

\begin{proof}
The statement of the main theorem is in figure $\ref{MainTheorem}$ and the details of the main theorm is in the appendix.

\begin{figure}
\begin{center}
\resizebox{4cm}{3cm}{
\begin{tikzpicture}[
roundnode/.style={circle, draw=green!60, fill=green!5, very thick, minimum size=7mm},
squarednode/.style={rectangle, draw=green!60, fill=green!5, very thick, minimum size=5mm},
]
\node[squarednode]      (history)                              {Time Series $\mathbf{X}_{n}$};
\node[squarednode]      (label)    [below=of history]       {Feedforward Neural Network $D$};
\node[squarednode]      (DNNTraining)    [below=of label]       {The Features of $\mathbf{X}_{n}$ in Table $\ref{FeatureEngineeringofTimeSeries}$};

\draw[->] (history.south) -- (label.north);
\draw[->] (label.south) -- (DNNTraining.north);

\end{tikzpicture}
}
\caption{The Statement of Main Theorem}\label{MainTheorem}
\end{center}
\end{figure}
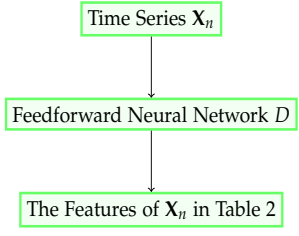

\end{proof}

From the statement of the main theorem, we can construct a feedforward neural network to calculate all features in the table $\ref{FeatureEngineeringofTimeSeries}$. Therefore we can also construct a feedforward neural network and train it from the normalized time series and labels. The original data contains three subsequences, the first part is before a week, the second part is in yesterday, the third part is in today. All of them combines together and get a sequence. The input of the neural network is a min-max-normalized subsequences, the output of the neural network is probabilities on 0 and 1, where 0 and 1 denote the anomaly points and normal points in time series, respectively. Therefore, we construct a feedforward neural network to train labelled datasets and do not need any feature engineering of time series.

In the industry, we usually collect a value per minute for each time series, then for every time series there are exactly $1440$ values per day and $10080$ values per week. Suppose there is a whole time series $[X_{1},\cdots, X_{n}]$ whose timestamp is more than two weeks (20160 minutes), i.e. $n\geq 20160$. In our platform,  we need to check three subsequences of the whole time series in order to label normal and abnormal points. More precisely, we want to check whether $x_{t}$ is anomaly or not for some timestamp $t$, we need all values between timestamps $t-k$ and $t+k$ of yesterday and last week. Using mathematical notations to write, we can get three subsequence from the whole time series $[X_{1},\cdots,X_{n}]$,
\begin{eqnarray*}
[z_{t-k}, \cdots, z_{t+k}], [y_{t-k},\cdots, y_{t+k}], [x_{t-k},\cdots,x_{t}],
\end{eqnarray*}
where
\begin{eqnarray*}
x_{i} &=& X_{i},  \mbox{ where }t-k\leq i\leq t,\\
y_{i} &=& X_{i-1440},  \mbox{ where }t-k\leq i\leq t+k,\\
z_{i} &=& X_{i-10080}, \mbox{ where }t-k\leq i\leq t+k.
\end{eqnarray*}

\begin{figure}
\begin{center}
\resizebox{7.5cm}{3.8cm}{
\begin{tikzpicture}
\scalebox{1}{
\draw[gray,thick,->] (0,0) -- (0,6);
\draw[gray,thick,->] (0,0) -- (10,0);
\draw[gray,thick,dashed] (5,0) -- (5,6);
\draw[gray,thick] (0,1) -- (10,1);
\draw[gray,thick] (0,3) -- (10,3);
\draw[gray,thick] (0,5) -- (10,5);
\draw[orange,ultra thick,<-] (2,1) -- node[above,sloped]{$k$ minutes ahead} (5,1);
\draw[orange,ultra thick,<-] (2,3) -- node[above,sloped]{$k$ minutes ahead} (5,3);
\draw[orange,ultra thick,<-] (2,5) -- node[above,sloped]{$k$ minutes ahead} (5,5);
\draw[orange,ultra thick,->] (5,3) -- node[above,sloped]{$k$ minutes later} (8,3);
\draw[orange,ultra thick,->] (5,5) -- node[above,sloped]{$k$ minutes later} (8,5);
\filldraw[black] (-1.7,1) node[anchor=west] {today};
\filldraw[black] (-1.7,3) node[anchor=west] {yesterday};
\filldraw[black] (-2.6,5) node[anchor=west] {one week ahead};
\filldraw[color=red!60, fill=red!5, very thick](5,1) circle (0.1);
\filldraw[black] (5.2,0.7) node[anchor=west] {Pending Point};
\filldraw[black] (1.5,0.7) node[anchor=west] {$x_{t-k}$};
\filldraw[black] (4.5,0.7) node[anchor=west] {$x_{t}$};
\filldraw[black] (1.5,2.7) node[anchor=west] {$y_{t-k}$};
\filldraw[black] (4.5,2.7) node[anchor=west] {$y_{t}$};
\filldraw[black] (7.5,2.7) node[anchor=west] {$y_{t+k}$};
\filldraw[black] (1.5,4.7) node[anchor=west] {$z_{t-k}$};
\filldraw[black] (4.5,4.7) node[anchor=west] {$z_{t}$};
\filldraw[black] (7.5,4.7) node[anchor=west] {$z_{t+k}$};
}
\end{tikzpicture}
}
\end{center}
\caption{Three Subsequences from Whole Time Series}
\label{ThreeSubsequencesFromWholeTimeSeries}
\end{figure}
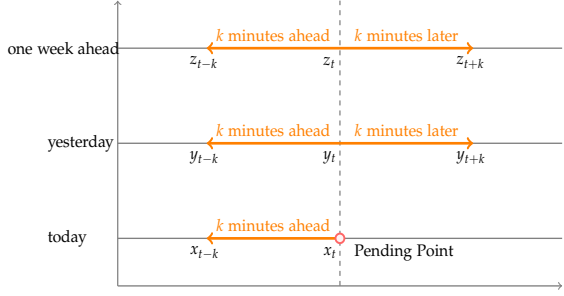

The details of three subsequences $x_{i}, y_{i},z_{i}$ is in figure $\ref{ThreeSubsequencesFromWholeTimeSeries}$, and the joint sequence of these three subsequences is 
\begin{eqnarray*}
[z_{t-k},\cdots,z_{t+k}, y_{t-k},\cdots,y_{t+k},x_{t-k},\cdots,x_{t}],
\end{eqnarray*}
and its length is $5k+3$. The min-max normalization of the joint sequence is
\begin{eqnarray*}
x_{i}' &=& \frac{x_{i}- a}{b-a},\mbox{ where }t-k\leq i\leq t,\\
y_{i}' &=& \frac{y_{i}-a}{b-a},\mbox{ where }t-k\leq i\leq t+k,\\
z_{i}' &=& \frac{z_{i}-a}{b-a},\mbox{ where }t-k\leq i\leq t+k.
\end{eqnarray*}
where
\begin{eqnarray*}
a &=& \min\{z_{t-k},\cdots,z_{t+k},y_{t-k},\cdots,y_{t+k},x_{t-k},\cdots,x_{t}\},\\
b &=& \max\{z_{t-k},\cdots,z_{t+k},y_{t-k},\cdots,y_{t+k},x_{t-k},\cdots,x_{t}\}.
\end{eqnarray*}
$[z_{t-k}',\cdots,z_{t+k}', y_{t-k}',\cdots,y_{t+k}',x_{t-k}',\cdots,x_{t}']$ is the joint sequence and the input of feedforward neural network. In the training and testing datasets, the value of $k$ is taken as $180$, which means $3$ hours ($180$ minutes) before and after the timestamp $t$. The length of joint sequence $[z_{t-k}',\cdots,z_{t+k}', y_{t-k}',\cdots,y_{t+k}',x_{t-k}',\cdots,x_{t}']$ is $5k+3 = 903$. For example, we want to know whether the value of some time series at the timestamp 10:00 am in $20180810$ is normal or not, then the three joint subsequences is drawn in figure \ref{ExamplesOfThreeSubsequencesfromWholeTimeSeries}.

\begin{figure}
\begin{center}
\resizebox{7.5cm}{3.8cm}{
\begin{tikzpicture}
\draw[gray,thick,->] (0,0) -- (0,6);
\draw[gray,thick,->] (0,0) -- (10,0);
\draw[gray,thick,dashed] (2,0) -- (2,6);
\draw[gray,thick,dashed] (5,0) -- (5,6);
\draw[gray,thick,dashed] (8,0) -- (8,6);
\draw[gray,thick] (0,1) -- (10,1);
\draw[gray,thick] (0,3) -- (10,3);
\draw[gray,thick] (0,5) -- (10,5);
\draw[orange,ultra thick,<-] (2,1) -- node[above,sloped]{$180$ minutes ahead} (5,1);
\draw[orange,ultra thick,<-] (2,3) -- node[above,sloped]{$180$ minutes ahead} (5,3);
\draw[orange,ultra thick,<-] (2,5) -- node[above,sloped]{$180$ minutes ahead} (5,5);
\draw[orange,ultra thick,->] (5,3) -- node[above,sloped]{$180$ minutes later} (8,3);
\draw[orange,ultra thick,->] (5,5) -- node[above,sloped]{$180$ minutes later} (8,5);
\filldraw[black] (-1.7,1) node[anchor=west] {$20180810$};
\filldraw[black] (-1.7,3) node[anchor=west] {$20180809$};
\filldraw[black] (-1.7,5) node[anchor=west] {$20180803$};
\filldraw[color=red!60, fill=red!5, very thick](5,1) circle (0.1);
\filldraw[black] (5.2,0.7) node[anchor=west] {Pending Point};
\filldraw[black] (1.1,0.7) node[anchor=west] {$x_{t-k}$};
\filldraw[black] (1.4,-0.3) node[anchor=west] {$07:00$};
\filldraw[black] (4.5,0.7) node[anchor=west] {$x_{t}$};
\filldraw[black] (4.4,-0.3) node[anchor=west] {$10:00$};
\filldraw[black] (1.1,2.7) node[anchor=west] {$y_{t-k}$};
\filldraw[black] (4.5,2.7) node[anchor=west] {$y_{t}$};
\filldraw[black] (8.0,2.7) node[anchor=west] {$y_{t+k}$};
\filldraw[black] (7.4,-0.3) node[anchor=west] {$13:00$};
\filldraw[black] (1.1,4.7) node[anchor=west] {$z_{t-k}$};
\filldraw[black] (4.5,4.7) node[anchor=west] {$z_{t}$};
\filldraw[black] (8.0,4.7) node[anchor=west] {$z_{t+k}$};
\end{tikzpicture}
}
\end{center}
\caption{Examples of Three Subsequences from Whole Time Series}
\label{ExamplesOfThreeSubsequencesfromWholeTimeSeries}
\end{figure}
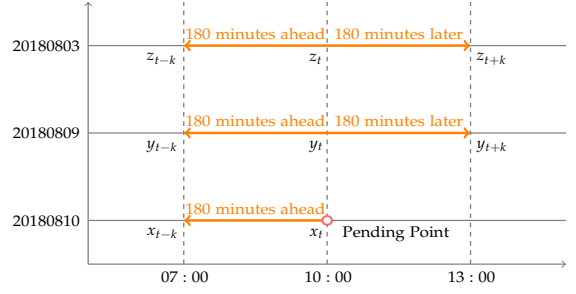

The networks between the input layer and the hidden layers, the hidden layers and the output layer are fully connected. The activation functions of these layers can be chosen as sigmoid function, $\tanh$ function, ReLU, leaky ReLU and so on. In mathematics, $\mbox{ReLU}(x) = \max\{0,x\}$ and $\mbox{leaky ReLU}(x) = \max\{x,\alpha x\}$. In fact, the parameter of  leaky ReLU is $\alpha = 0.2$ in tensorflow. In our work, leaky ReLU with default parameter is used as activation function in feedforward neural network.

The output of the neural network is the probability of two classes, where $0$ and $1$ denote anomaly and normal, respectively. The activation function of the output layer is softmax function. More precisely, if we assume the input of the last layer is $\alpha_{0},\alpha_{1}$, the output of the last layer is
\begin{eqnarray*}
\beta_{i} = \frac{e^{\alpha_{i}}}{e^{\alpha_{0}}+e^{\alpha_{1}}}, \mbox{ where } 0\leq i\leq 1,
\end{eqnarray*}
where $\beta_{0}$ and $\beta_{1}$ denote the probability of label $0$ and $1$ for some time series sample, respectively. The loss function of feedforward neural network is cross entropy.

In reality, the number of anomaly cases are much less than normal cases. For the most of the day, there are no anomaly points for many time series, that means the labelled data is imbalance. To overcome difficulties, there are two basic methods in machine learning theory, under sampling and over sampling. In this situation, we use under sampling to solve class imbalance problem. That means we choose all anomaly cases in labelled data, and randomly choose some normal cases in labelled data such that the ratio of anomaly and normal samples is about $2:1$. The format of labelled data is described in table $\ref{LabelsOfMin-MaxNormalizedTimeSeries}$, which contains labels and subsequences of time series. The details of training datasets and test datasets are in table $\ref{DatasetsOfTimeSeriesAnomalyDetection}$.

The number of parameters of neural network is about $100$ thousands. If we have labelled samples less than $10$ thousands, then we can not train an excellent neural network model, but we can train supervised models from random forest algorithm or GBDT algorithm through feature engineering of time series. Therefore, in the situation, if we want to use deep feedforward neural network to ignore the feature engineering, then we must provide enough labelled data, which is at least $50$ thousands.

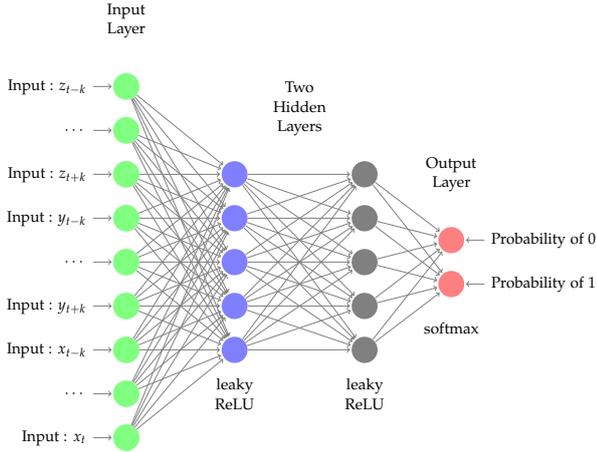
\begin{figure}
\begin{center}
\resizebox{8cm}{6cm}{
\begin{tikzpicture}[shorten >=1pt,->,draw=black!50, node distance=\layersep]
    \tikzstyle{every pin edge}=[<-,shorten <=1pt]
    \tikzstyle{neuron}=[circle,fill=black!25,minimum size=17pt,inner sep=0pt]
    \tikzstyle{input neuron}=[neuron, fill=green!50];
    \tikzstyle{output neuron}=[neuron, fill=red!50];
    \tikzstyle{hidden neuron1}=[neuron, fill=blue!50];
    \tikzstyle{hidden neuron2}=[neuron, fill=black!50];
    \tikzstyle{annot} = [text width=4em, text centered]

    \foreach \name / \y in {1,...,9}
        \node[input neuron] (I-\name) at (0,-\y) {};

    \node[input neuron, pin=left:Input : $z_{t-k}$] (I-1) at (0,-1) {};
    \node[input neuron, pin=left: $\cdots$] (I-2) at (0,-2) {};
    \node[input neuron, pin=left:Input : $z_{t+k}$] (I-3) at (0,-3) {};
    \node[input neuron, pin=left:Input : $y_{t-k}$] (I-4) at (0,-4) {};
    \node[input neuron, pin=left:$\cdots$] (I-5) at (0,-5) {};
    \node[input neuron, pin=left:Input : $y_{t+k}$] (I-6) at (0,-6) {};
    \node[input neuron, pin=left:Input : $x_{t-k}$] (I-7) at (0,-7) {};
    \node[input neuron, pin=left:$\cdots$] (I-8) at (0,-8) {};
    \node[input neuron, pin=left:Input : $x_{t}$] (I-9) at (0,-9) {};

    \foreach \name / \y in {1,...,5}
        \path[xshift=0.0cm,yshift = -2.0cm]
            node[hidden neuron1] (H1-\name) at (\layersep,-\y cm) {};

    \foreach \name / \y in {1,...,5}
        \path[xshift=3.0cm,yshift = -2.0cm]
            node[hidden neuron2] (H2-\name) at (\layersep,-\y cm) {};

    \foreach \name / \y in {0,...,1}
        \path[xshift=5.0cm,yshift = -4.5cm]
            node[output neuron,pin=right:Probability of \y] (O-\name) at (\layersep,-\y cm) {};

    \foreach \source in {1,...,9}
        \foreach \dest in {1,...,5}
            \path (I-\source) edge (H1-\dest);

    \foreach \source in {1,...,5}
        \foreach \dest in {1,...,5}
            \path (H1-\source) edge (H2-\dest);

    \foreach \source in {1,...,5}
        \foreach \dest in {0,...,1}
            \path (H2-\source) edge (O-\dest);

    \node[annot, above of=H1-1, right of =H1-1, node distance=1.5cm] (hl) {Two Hidden Layers};
    \node[annot, above of=I-2] {Input Layer};
    \node[annot, above of=O-1] {Output Layer};
    \node[annot, below of=H1-5, node distance=1cm] {leaky ReLU};
    \node[annot, below of=H2-5, node distance=1cm] {leaky ReLU};
    \node[annot, below of=O-1, node distance=1cm] {softmax};
\end{tikzpicture}
}
\caption{The Architecture of Deep Feedforward Neural Network}\label{TheArchitectureofDeepFeedforwardNeuralNetwork}
\end{center}
\end{figure}

\begin{table}
\begin{center}
\begin{tabular}{ |m{4em}|  m{4em}| m{13.5em} | }
\hline
No. of Samples & Labels & Min-Max Normalized Time Series \\
\hline
1 & Anomaly & $\cdots$ \\
\hline
2 & Anomaly & $\cdots$  \\
\hline
3 & Normal & $\cdots$  \\
\hline
4 & Normal & $\cdots$  \\
\hline
$\cdots$ & $\cdots$ & $\cdots$ \\
\hline
\end{tabular}
\caption{Labels of Min-Max Normalized Time Series}
\label{LabelsOfMin-MaxNormalizedTimeSeries}
\end{center}
\end{table}

\section{Evaluation}

\subsection{Datasets}
In order to implement the system, we use python and its open source library, such as scikit learn, XGBoost, tensorflow and tsfresh, which is an open source of time series feature extraction tool. In order to compare the method between XGBoost and deep feedforward neural network, we need to prepare four files, which contain negative and positive training datasets, negative and positive testing datasets, and the intersection of any two files are empty. The number of these files are in table $\ref{DatasetsOfTimeSeriesAnomalyDetection}$. Every sample in training and test datasets contains 903 points and its label, which is labelled by human.

\begin{table}
\begin{center}
\begin{tabular}{ | m{10em} | m{4em}| m{4em} | }
\hline
Datasets& Negative & Positive \\
\hline
Training Datasets& 48986 & 29134 \\
\hline
Test Datasets& 4509 & 11226 \\
\hline
\end{tabular}
\caption{Datasets of Time Series Anomaly Detection}
\label{DatasetsOfTimeSeriesAnomalyDetection}
\end{center}
\end{table}

\subsection{Performance Metrics and Experiments}
The details of deep feedforward neural network are described in figure $\ref{TheArchitectureofDeepFeedforwardNeuralNetwork}$ and figure $\ref{TheConstructionOfFeedforwardNeuralNetwork}$. In the neural network, the number of hidden layers is two, both the activation functions of two hidden layers are leaky ReLU with default parameters, and the activation function of the output layer is the softmax function. The structure of feedforward neural network is drawn in figure $\ref{TheConstructionOfFeedforwardNeuralNetwork}$, which is drawn by TensorBoard.
\begin{figure*}[ht]
\centering
\includegraphics[width=\textwidth]{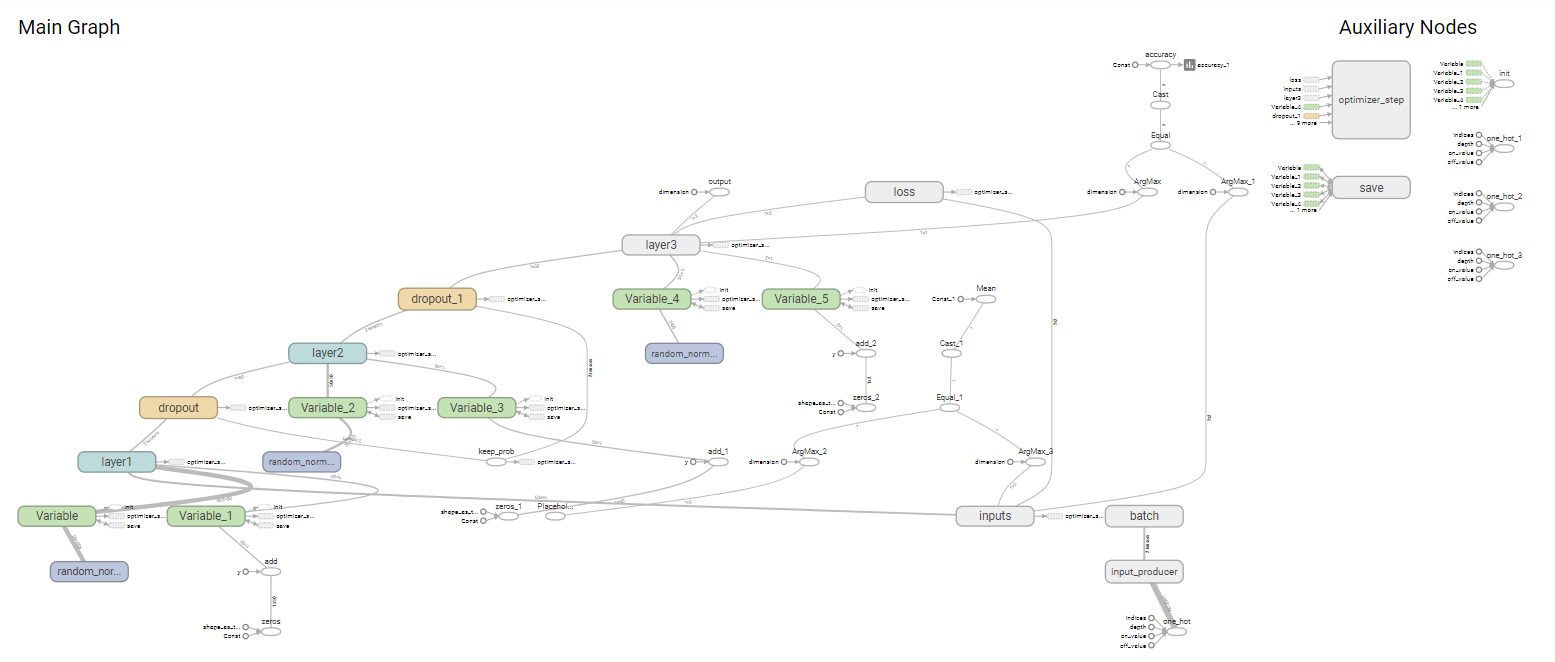}
\caption{The Construction of Feedforward Neural Network}
\label{TheConstructionOfFeedforwardNeuralNetwork}
\end{figure*}

In this paper, we use these indices in section 2.2 to evaluate the unsupervised and supervised models, such as control chart theory, isolation forest, polynomial regression, XGBoost and deep feedforward neural networks. The parameters of these models are in table $\ref{ParametersOfModels}$ and table $\ref{Experiments}$ is the experiment results. In DNN model, the graphs between loss, accuracy and iterations are in figure $\ref{LossAndAccuracy}$.

\begin{figure}[ht]
\centering
\includegraphics[width=0.4\textwidth]{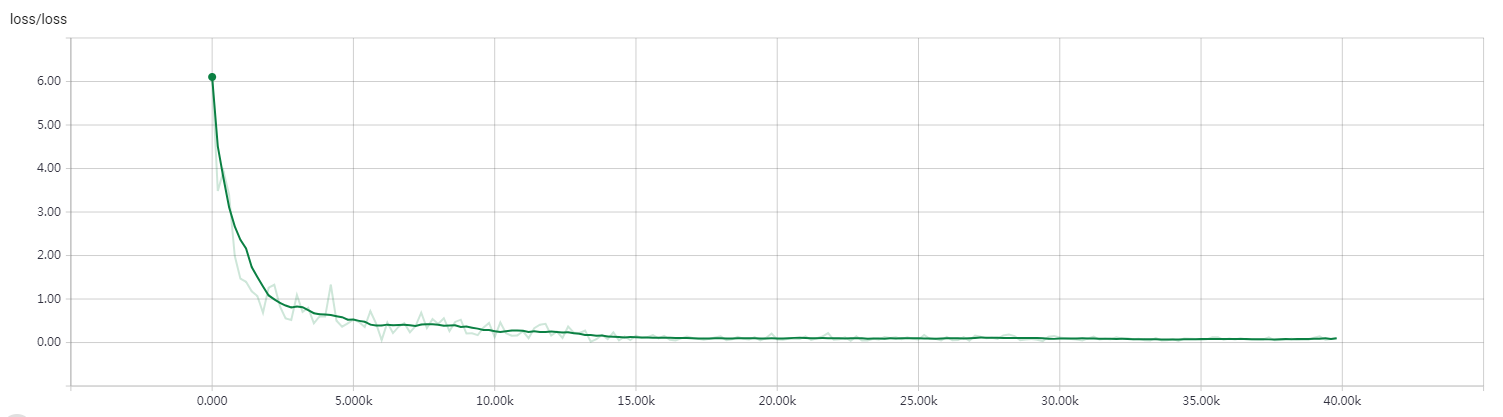}
\includegraphics[width=0.4\textwidth]{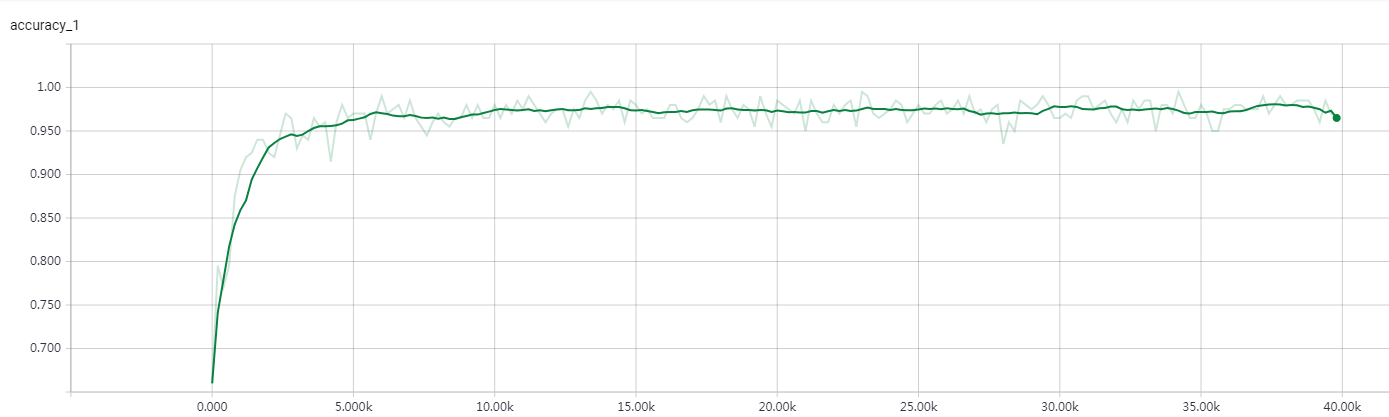}
\caption{The Loss and Accuracy of DNN}
\label{LossAndAccuracy}
\end{figure}

\begin{table}
\begin{center}
\newcommand{\tabincell}[2]{\begin{tabular}{@{}#1@{}}#2\end{tabular}}
 \begin{tabular}{| m{5em}  | m{17em} |}
\hline
Algorithms & Parameters\\
\hline
3-Sigma & None\\
\hline
EWMA Control Chart& coefficient = 3, alpha = 0.3\\
\hline
Polynomial Regression & degree = 4, threshold = 0.3\\
\hline
Isolation Forest & nestimator = 3, maxsample = 'auto',  contamination = 0.15\\
\hline
XGBoost & 243 features,  maxdepth = 10, eta = 0.05, gamma = 0.1, booster = 'gbtree', objective = 'binary:logistic', evalmetric = 'auc'\\
\hline
DNN-1 & number of hidden layers = 2,  number of elements in hidden layer = 50, activation function = Leaky ReLU,  no dropout\\
\hline
DNN-2 & number of hidden layers = 2,  number of elements in hidden layer = 50, activation function = Leaky ReLU,  dropout = 0.95\\
\hline
\end{tabular}
\caption{Parameters of Models}\label{ParametersOfModels}
\end{center}
\end{table}

\begin{table*}
\begin{center}
\begin{tabular}{ | m{9.5em} | m{0.8cm}| m{0.8cm} |m{0.8cm}| m{0.8cm} | m{1.0cm}| m{1.2cm} | m{1.3cm} |}
\hline
Algorithms & TP & FN & FP & TN & Recall & Precision & F1-Score\\
\hline
3-Sigma & 10524 & 702 & 1331 & 3178 & 70.5\% & 81.9\% &75.8\%\\
\hline
EWMA Control Chart & 8981 & 2245  & 294  & 4215  &  93.5\% &  65.2\% &   76.8\%\\
\hline
Polynomial Regression & 9193 & 2033 & 579 & 3930 & 87.2\% & 65.9\% & 75.7\%\\
\hline
Isolation Forest & 9071 & 2155 & 364 & 4145 & 91.2\% & 65.8\% & 76.4\%\\
\hline
XGBoost & 11082 & 144 & 1039 & 3470 & 77.0\% & 96.0\% & 85.5\%\\
\hline
DNN-1 & 11091 & 135 & 891 & 3618 & 80.2\% & 96.4\% & 87.6\%\\
\hline
DNN-2 &  11100  &  126 & 852  & 3657  &  81.1\% &  96.7\% &  88.2\%\\
\hline
\end{tabular}
\caption{Experiments of Algorithms}\label{Experiments}
\end{center}
\end{table*}

\subsection{Output of Hidden Layers as Features of Time Series}
\subsubsection{Clustering of a Set of Time Series}
In statistics, there are several clustering algorithms of clustering of time series, such as KMeans and hierarchical clustering. For example, pearson coefficient and dynamic time warping can be used to calculated the similarity between two time series, then we can get clusters of several time series based on similarity measures and feature engineerings of time series. In the paper, we use a trained neural network to calculate features of time series. More precisely, in our deep feedforward neural network, there are two hidden layers, the output of these two layers can be seen as the hidden features of time series, which is called "Time Series To Vector". For a trained neural network, we can get two hidden features from one time series, then we use clustering algorithms such as KMeans to get several clusterings. The figure \ref{ClustersOfTimeSeries} is the clustering result of time series with these two hidden features.

\begin{figure}[ht]
\centering
\includegraphics[width=0.4\textwidth]{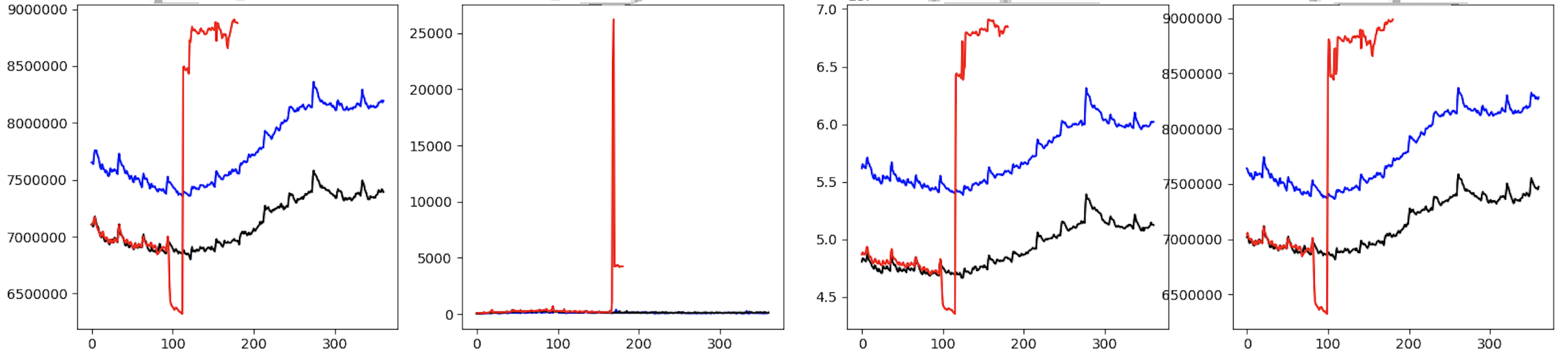}
\includegraphics[width=0.4\textwidth]{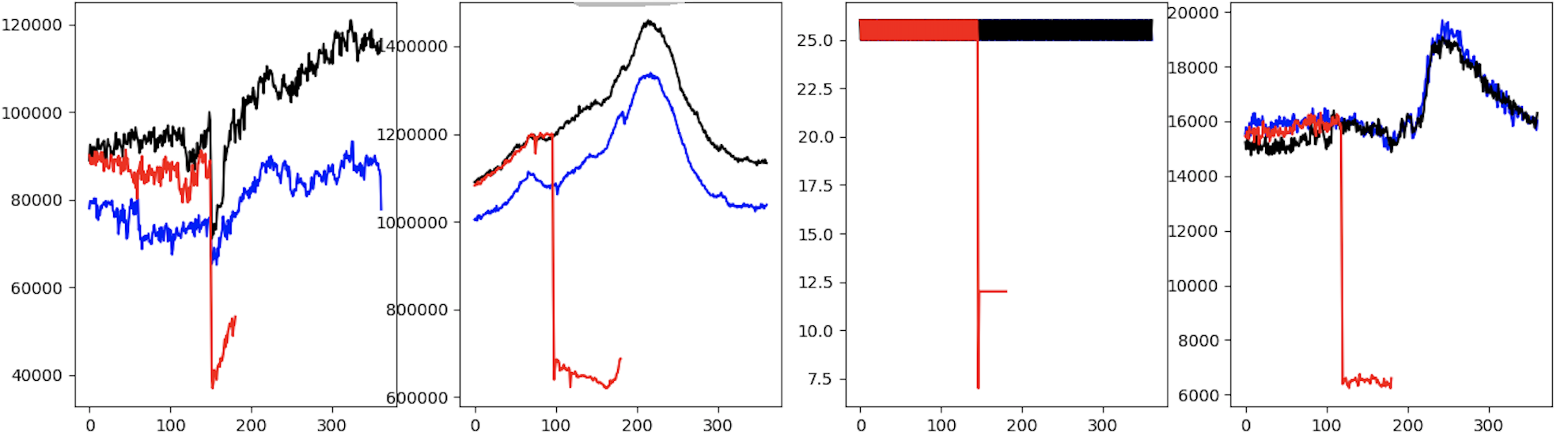}
\caption{Clusters of Time Series}
\label{ClustersOfTimeSeries}
\end{figure}

\subsubsection{Cosine Similarity between Two Time Series}
In mathematics, there are many formulas to calculate similarity between two different time series, such as pearson coefficient, SAX, piecewise average aggregation. In this paper, we can also use the output of two hidden layers as the features of time series, then the cosine similarity between the output of hidden layers of two time series is a similarity measure between them. The performance of this similarity method can be seen in figure \ref{ThreeSimilarTimeSeriesOfTheFirstTimeSeries}.

\begin{figure}[ht]
\centering
\includegraphics[width=0.4\textwidth]{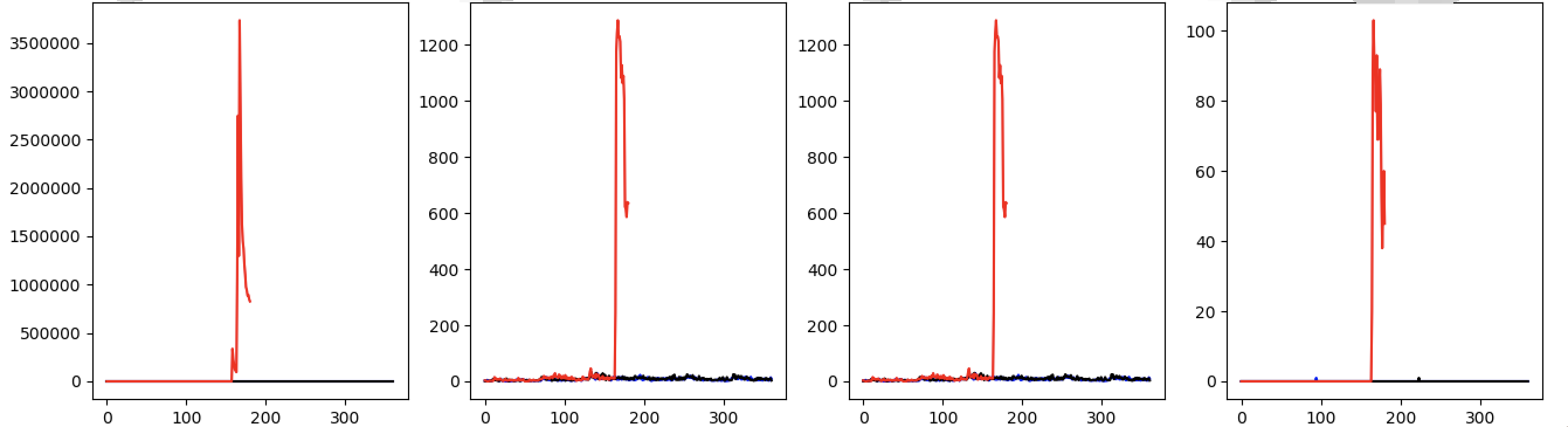}
\includegraphics[width=0.4\textwidth]{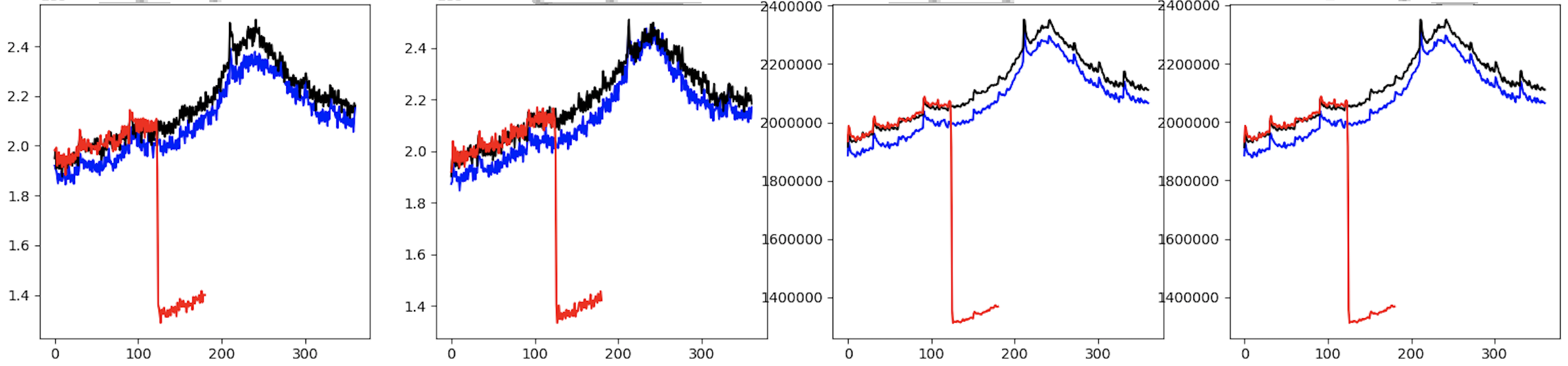}
\includegraphics[width=0.4\textwidth]{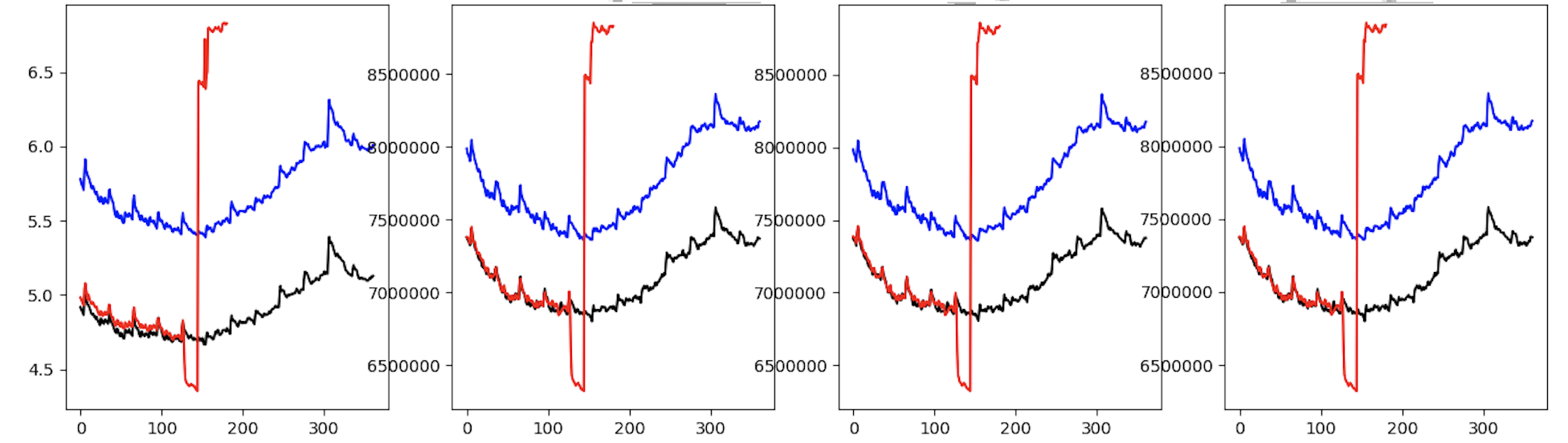}
\caption{Three Similar Time Series of the First Time Series Every Line}
\label{ThreeSimilarTimeSeriesOfTheFirstTimeSeries}
\end{figure}

\section{Conclusion}
In conclusion, from the main theorem and experiments, deep feedforward neural network is a useful supervised model in time series anomaly detection, and it can get hidden features of time series from a trained neural network. Training a deep neural network is an end-to-end method which need only raw datasets and their corresponding labels. The performance metrics, such as precision, recall and F1-score, of deep feedforward neural network is better than XGBoost with feature engineering. The most important thing of the paper is that we do not need feature engineering in the process of training a feedforward neural network. After we get a trained neural network, the output of hidden layers of some time series can be seen as the hidden features of them, which can be called as "TimeSeries2Vec". Then we can use cosine similarity to calculate the similarity between two different time series or do clustering analysis from these hidden features. To the best of our knowledge, our system is the first anomaly detection framework based on deep neural network, and it do not need any feature engineering of time series, which is the biggest difference between ours and other supervised models.

\appendix
\section{Appendix}
In the section, we pay attention to prove the main theorem of the paper. The main theorem is as follows:

\begin{thm}
Suppose $n\geq 1$ is a positive integer, there exists a deep feedforward neural network $D$ such that for any real time series $\mathbf{X}_{n}=[X_{1},\cdots, X_{n}]$, the input and output of $D$ are  $\mathbf{X}_{n}$ and the features of $\mathbf{X}_{n}$ in Table $\ref{FeatureEngineeringofTimeSeries}$, respectively.
\end{thm}

\begin{proof}
At the beginning, we will show that some basic calculations, such as addition, subtraction, absolute value, maximum, minimum and average, can be construct from feedforward neural networks. More precisely,
\begin{eqnarray*}
\text{add}(x,y) &=& x+y, \\
\text{sub}(x,y) &=& x-y, \\
\text{abs}(x) &=& ReLU(x) + ReLU(-x),\\
\max(x,y) &=& (x+y+|x-y|)/2,\\
\min(x,y) &=& (x+y-|x-y|)/2,\\
\text{average}(x_{1},\cdots,x_{n}) &=& (x_{1}+\cdots+x_{n})/n,
\end{eqnarray*}
where $\text{ReLU}(x) = \max(x,0)$. Since neural network contains matrix summation and multiplication, we easily see the above functions can be written as the form of feedforward neural networks in figure $\ref{BasicCalculationofRealNumbers}$.

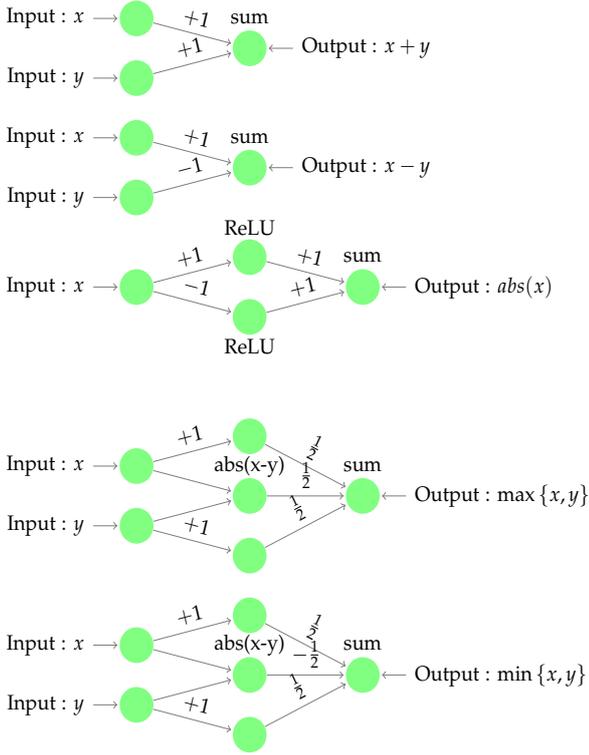
\begin{figure}
\begin{center}
\resizebox{8cm}{10cm}{
\begin{tikzpicture}[shorten >=1pt,->,draw=black!50, node distance=\layersep]
    \tikzstyle{every pin edge}=[<-,shorten <=1pt]
    \tikzstyle{neuron}=[circle,fill=black!25,minimum size=17pt,inner sep=0pt]
    \tikzstyle{input neuron}=[neuron, fill=green!50];
    \tikzstyle{output neuron}=[neuron, fill=red!50];
    \tikzstyle{hidden neuron1}=[neuron, fill=blue!50];
    \tikzstyle{hidden neuron2}=[neuron, fill=black!50];
    \tikzstyle{annot} = [text width=4em, text centered]

    \node[input neuron, pin=left:Input : $x$] (I-1) at (0,-1) {};
    \node[input neuron, pin=left:Input : $y$] (I-2) at (0,-2) {};
    \node[input neuron, pin=right:Output : $x+y$] (R-1) at (2,-1.5) {};
    \path (I-1) edge node[above,sloped]{$+1$} (R-1);
    \path (I-2) edge node[above,sloped]{$+1$} (R-1);
    \node at (2,-1)[ align=center]{sum};

    \node[input neuron, pin=left:Input : $x$] (I-3) at (0,-3) {};
    \node[input neuron, pin=left:Input : $y$] (I-4) at (0,-4) {};
    \node[input neuron, pin=right:Output : $x-y$] (R-2) at (2,-3.5) {};
    \path (I-3) edge node[above,sloped]{$+1$} (R-2);
    \path (I-4) edge node[above,sloped]{$-1$} (R-2);
    \node at (2,-3)[ align=center]{sum};

    \node[input neuron, pin=left:Input : $x$] (J-1) at (0,-5.5) {};
    \node[input neuron] (I-5) at (2,-5) {};
    \node[input neuron] (I-6) at (2,-6) {};
    \node[input neuron, pin=right:Output : $abs(x)$] (R-3) at (4,-5.5) {};
    \path (J-1) edge node[above,sloped]{$+1$} (I-5);
    \path (J-1) edge node[above,sloped]{$-1$} (I-6);
    \path (I-5) edge node[above,sloped]{$+1$} (R-3);
    \path (I-6) edge node[above,sloped]{$+1$} (R-3);
    \node at (4,-5)[ align=center]{sum};
    \node at (2,-4.5)[ align=center]{ReLU};
    \node at (2,-6.5)[ align=center]{ReLU};

    \node[input neuron, pin=left:Input : $x$] (J-2) at (0,-8.5) {};
    \node[input neuron, pin=left:Input : $y$] (J-3) at (0,-9.5) {};
    \node[input neuron] (I-7) at (2,-8) {};
    \node[input neuron] (I-8) at (2,-9) {};
    \node[input neuron] (I-9) at (2,-10) {};
    \node[input neuron, pin=right:Output : $\max{\{x,y\}}$] (R-4) at (4,-9) {};
    \path (J-2) edge node[above,sloped]{$+1$} (I-7);
    \path (J-2) edge node[above,sloped]{$$} (I-8);
    \path (J-3) edge node[above,sloped]{$+1$} (I-9);
    \path (J-3) edge node[above,sloped]{$$} (I-8);
    \path (I-7) edge node[above,sloped]{$\frac{1}{2}$} (R-4);
    \path (I-8) edge node[above,sloped]{$\frac{1}{2}$} (R-4);
    \path (I-9) edge node[above,sloped]{$\frac{1}{2}$} (R-4);
    \node at (4,-8.5)[ align=center]{sum};
    \node at (2,-8.5)[ align=center]{abs(x-y)};

    \node[input neuron, pin=left:Input : $x$] (J-4) at (0,-11.5) {};
    \node[input neuron, pin=left:Input : $y$] (J-5) at (0,-12.5) {};
    \node[input neuron] (I-10) at (2,-11) {};
    \node[input neuron] (I-11) at (2,-12) {};
    \node[input neuron] (I-12) at (2,-13) {};
    \node[input neuron, pin=right:Output : $\min{\{x,y\}}$] (R-5) at (4,-12) {};
    \path (J-4) edge node[above,sloped]{$+1$} (I-10);
    \path (J-4) edge node[above,sloped]{$$} (I-11);
    \path (J-5) edge node[above,sloped]{$+1$} (I-12);
    \path (J-5) edge node[above,sloped]{$$} (I-11);
    \path (I-10) edge node[above,sloped]{$\frac{1}{2}$} (R-5);
    \path (I-11) edge node[above,sloped]{$-\frac{1}{2}$} (R-5);
    \path (I-12) edge node[above,sloped]{$\frac{1}{2}$} (R-5);
    \node at (4,-11.5)[ align=center]{sum};
    \node at (2,-11.5)[ align=center]{abs(x-y)};

\end{tikzpicture}
}
\caption{Basic Calculation of Real Numbers}\label{BasicCalculationofRealNumbers}
\end{center}
\end{figure}

From above, we have already known that some basic calculation between two real numbers can be constructed by feedforward neural network. Next we will prove that these features in Table $\ref{FeatureEngineeringofTimeSeries}$ can also be constructed by feedforward neural networks. Suppose there is a time series $X_{n} = [x_{1},\cdots, x_{n}]$ with length $n$, some features are defined as followings:

\begin{eqnarray*}
& &\max : \\
& &\max_{1\leq i\leq n}\{x_{1},\cdots,x_{n}\} \\
& &=  \max\{x_{1},\max\{x_{2},\cdots,\max\{x_{n-1},x_{n}\}\}\}, \\
& &\min : \\
& &\min_{1\leq i\leq n}\{x_{1},\cdots,x_{n}\} \\
& &= \min\{x_{1},\min\{x_{2},\cdots,\min\{x_{n-1},x_{n}\}\}\}, \\
& &\mbox{average} :  \sum_{i=1}^{n}x_{i}/n, \\
& &\mbox{difference} : x_{2}-x_{1},\cdots,x_{n}-x_{n-1},\\
& &\mbox{integration} : \sum_{i=1}^{n}x_{i}, \\
& &\mbox{absolute sum of changes} : \sum_{i=1}^{n-1}|x_{i+1}-x_{i}|,\\
& &\mbox{mean change} : \sum_{i=1}^{n-1}(x_{i+1}-x_{i})/n = (x_{n}-x_{1})/n,\\
& &\mbox{mean second derivative central} : \\
& &\sum_{i=1}^{n-2}(x_{i+2}-2x_{i+1}+x_{i})/(2n).
\end{eqnarray*}

From basic calculation between real numbers, the above features can be also constructed by feedforward neural networks. Recall that Simple Moving Average, Weighted Moving Average, Exponentially Weighted Moving Average algorithms are defined as
\begin{eqnarray*}
SMA_{n}(w) 
&=& \frac{\sum_{k=1}^{w} x_{n-w+1}}{w} \\
&=& \frac{x_{n-w+1}+\cdots+x_{n}}{w},\\
WMA_{n}(w) &=& \frac{\sum_{k=1}^{w} k\cdot x_{n-w+k}}{\sum_{k=1}^{w}k} \\
&=& \frac{2\cdot\sum_{k=1}^{w} k\cdot x_{n-w+k}}{w(w+1)}, \\
EWMA_{j}(\alpha)  & =& x_{1}, \mbox{ if } j= 1,\\
EWMA_{j}(\alpha)  & =& \alpha\cdot x_{j-1}+(1-\alpha)\cdot EWMA_{j-1}, \mbox{ if } j\geq 2.
\end{eqnarray*}
where $w\geq 1$ is the window size and $\alpha\in[0,1]$ is a factor. The fitting features from these statistical models are constructed by the subtraction between fitting values from these models and real values. More precisely, the formulas of fitting features are
\begin{eqnarray*}
SMA_{n}(w) -x_{n}, \mbox{ where } w = 10, 20, 30, 40, 50,\\
WMA_{n}(w) - x_{n}, \mbox{ where } w = 10, 20, 30, 40, 50,\\
EWMA_{n}(\alpha) -x_{n}, \mbox{ where } \alpha = 0.2, 0.4, 0.6, 0.8.
\end{eqnarray*}
From its definition, EWMA can be written as
\begin{eqnarray*}
& &EWMA_{n}(\alpha) \\
&=& \alpha\cdot x_{n-1} + (1-\alpha)\cdot EWMA_{n-1} \\
&=& \alpha \cdot x_{n-1} + \alpha(1-\alpha) x_{n-2} + \alpha(1-\alpha) EWMA_{n-2} \\
&=& \alpha \cdot x_{n-1} + \alpha(1-\alpha) x_{n-2} +\cdots + \alpha^{n-2}(1-\alpha)x_{1}.
\end{eqnarray*}
In figure $\ref{FittingFeaturesOfTimeSeries}$, we have already shown that how to construct fitting features of time series from simple moving average, weighted moving average, exponentially moving average algorithms.

\begin{figure}
\begin{center}
\resizebox{8.3cm}{9.5cm}{
\begin{tikzpicture}[shorten >=1pt,->,draw=black!50, node distance=\layersep]
    \tikzstyle{every pin edge}=[<-,shorten <=1pt]
    \tikzstyle{neuron}=[circle,fill=black!25,minimum size=17pt,inner sep=0pt]
    \tikzstyle{input neuron}=[neuron, fill=green!50];
    \tikzstyle{output neuron}=[neuron, fill=red!50];
    \tikzstyle{hidden neuron1}=[neuron, fill=blue!50];
    \tikzstyle{hidden neuron2}=[neuron, fill=black!50];
    \tikzstyle{annot} = [text width=4em, text centered]

    \node[input neuron, pin=left:Input : $x_{n-w+1}$] (I-1) at (0,-1) {};
    \node[input neuron, pin=left:Input : $\cdots$] (I-2) at (0,-2) {};
    \node[input neuron, pin=left:Input : $x_{n}$] (I-3) at (0,-3) {};
    \node[input neuron ] (R-1) at (3,-1.5) {};
    \node[input neuron ] (R-2) at (3,-3) {};
    \path (I-1) edge node[above,sloped]{$\frac{1}{w}$} (R-1);
    \path (I-2) edge node[above,sloped]{$\frac{1}{w}$} (R-1);
    \path (I-3) edge node[above,sloped]{$\frac{1}{w}$} (R-1);
    \path (I-3) edge node[above,sloped]{$+1$} (R-2);
    \node at (3,-1)[ align=center]{$SMA_{n}(w)$};
    \node at (3,-3.5)[ align=center]{$x_{n}$};

    \node[input neuron, pin=right:Output : $SMA_{n}(w)-x_{n}$] (J-1) at (6,-2) {};
    \path (R-1) edge node[above,sloped]{$+1$} (J-1);
    \path (R-2) edge node[above,sloped]{$-1$} (J-1);

    ----
    \node[input neuron, pin=left:Input : $x_{n-w+1}$] (A-1) at (0,-5) {};
    \node[input neuron, pin=left:Input : $\cdots$] (A-2) at (0,-6) {};
    \node[input neuron, pin=left:Input : $x_{n}$] (A-3) at (0,-7) {};
    \node[input neuron ] (B-1) at (3,-5.5) {};
    \node[input neuron ] (B-2) at (3,-7) {};
    \path (A-1) edge node[above,sloped]{$\frac{2}{w(w+1)}$} (B-1);
    \path (A-2) edge node[above,sloped]{$$} (B-1);
    \path (A-3) edge node[above,sloped]{$\frac{2w}{w(w+1)}$} (B-1);
    \path (A-3) edge node[above,sloped]{$+1$} (B-2);
    \node at (4,-5)[ align=center]{$WMA_{n}(w)$};
    \node at (4,-7)[ align=center]{$x_{n}$};

    \node[input neuron, pin=right:Output : $WMA_{n}(w)-x_{n}$] (C-1) at (6,-6) {};
    \path (B-1) edge node[above,sloped]{$+1$} (C-1);
    \path (B-2) edge node[above,sloped]{$-1$} (C-1);

    ----
    \node[input neuron, pin=left:Input : $x_{1}$] (E-1) at (0,-9) {};
    \node[input neuron, pin=left:Input : $\cdots$] (E-2) at (0,-10) {};
    \node[input neuron, pin=left:Input : $x_{n-2}$] (E-3) at (0,-11) {};
    \node[input neuron, pin=left:Input : $x_{n-1}$] (E-4) at (0,-12) {};
    \node[input neuron, pin=left:Input : $x_{n}$] (E-5) at (0,-13) {};
    \node[input neuron ] (F-1) at (3,-10.5) {};
    \node[input neuron ] (F-2) at (3,-13) {};
    \path (E-1) edge node[above,sloped]{$\alpha^{n-2}(1-\alpha)$} (F-1);
    \path (E-2) edge node[above,sloped]{$$} (F-1);
    \path (E-3) edge node[above,sloped]{$(1-\alpha)\alpha$} (F-1);
    \path (E-4) edge node[above,sloped]{$\alpha$} (F-1);
    \path (E-5) edge node[above,sloped]{$+1$} (F-2);
    \node at (3.5,-9.5)[ align=center]{$EWMA_{n}(w)$};
    \node at (3.5,-12)[ align=center]{$x_{n}$};

    \node[input neuron, pin=right:Output : $EWMA_{n}(w)-x_{n}$] (G-1) at (6,-11) {};
    \path (F-1) edge node[above,sloped]{$+1$} (G-1);
    \path (F-2) edge node[above,sloped]{$-1$} (G-1);

\end{tikzpicture}
}
\caption{Fitting Features of Time Series}\label{FittingFeaturesOfTimeSeries}
\end{center}
\end{figure}
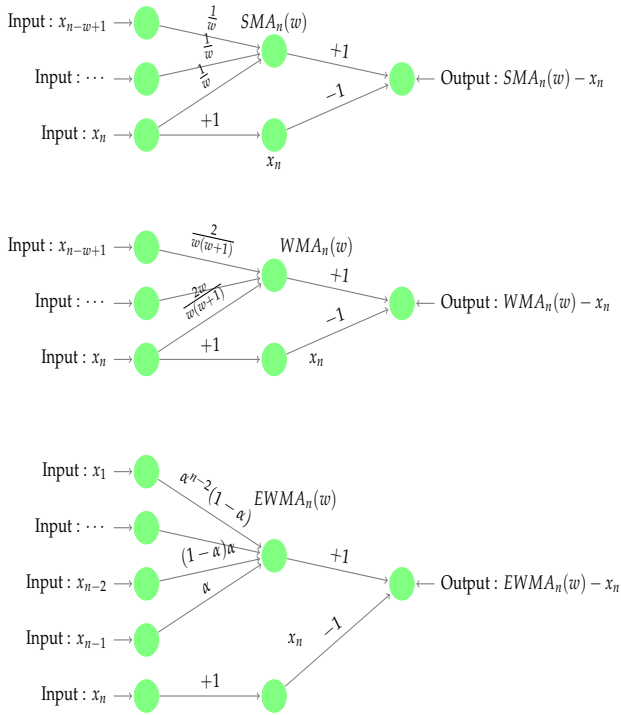

Next, we will prove that simple threshold, count above mean and count below mean can be written as the form of feedforward neural network. First, the features of simple threshold for a fixed real number $a$ are defined as, i.e.

\begin{equation*}
\mathcal{I}_{\{x\geq a\}} =\begin{cases}
&1, \mbox{ if } x\geq a,\\
&0, \mbox{ else } x<a.
\end{cases}
\end{equation*}
and
\begin{equation*}
\mathcal{I}_{\{x<a\}} =\begin{cases}
&0, \mbox{ if } x\geq a,\\
&1, \mbox{ else } x<a.
\end{cases}
\end{equation*}

Let
\begin{eqnarray*}
f_{a}(x) &=& \sigma(-2\cdot 10^{4}\cdot ReLU(-x+a) + 10),\\
g_{a}(x) &=& \sigma(-2\cdot 10^{4}\cdot ReLU(x-a) + 10),
\end{eqnarray*}
where $\sigma$ denotes the sigmoid function, i.e. $\sigma(x)=1/(1+\exp(-x))$. Then if $x>a$, then $f_{a}(x) = \sigma(10) \approx 1;$ if $x<a-10^{3}$, then $f_{a}(x) =\sigma(-2\cdot 10^{4}\cdot(a-x)+10)<\sigma(-10)\approx 0$. Therefore, $f_{a}(x) \approx \mathcal{I}_{\{x\geq a\}}$. Similarly, the proof of $g_{a}(x)\approx \mathcal{I}_{\{x<a\}}$ is the same as before.

Second, the construction of count above mean and count below mean are based on $f_{a}(x)$ and $g_{a}(x)$. More precisely, count above mean denotes the number of the time series $[x_{1},\cdots,x_{n}]$ which are larger than the average value of time series. Similarly, count below mean denotes the number of the time series $[x_{1},\cdots,x_{n}]$ which are less than the average value of time series. In figure $\ref{CountFeaturesOfTimeSeries}$, we have already shown that how to construct these two features from $f_{a}(x)$ and $g_{a}(x)$. Hence, we have already proven the main theorem of the paper.

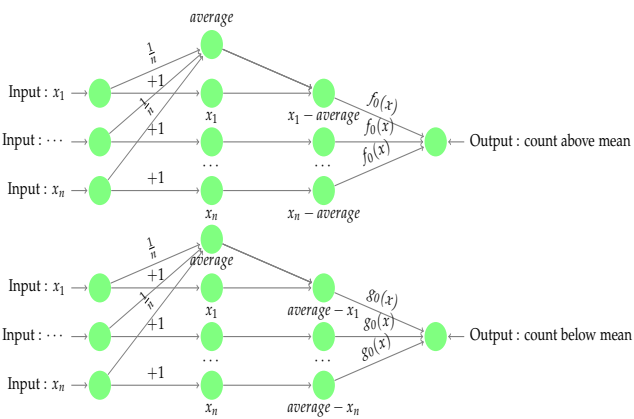
\begin{figure}
\begin{center}
\resizebox{8.5cm}{5.5cm}{
\begin{tikzpicture}[shorten >=1pt,->,draw=black!50, node distance=\layersep]
    \tikzstyle{every pin edge}=[<-,shorten <=1pt]
    \tikzstyle{neuron}=[circle,fill=black!25,minimum size=17pt,inner sep=0pt]
    \tikzstyle{input neuron}=[neuron, fill=green!50];
    \tikzstyle{output neuron}=[neuron, fill=red!50];
    \tikzstyle{hidden neuron1}=[neuron, fill=blue!50];
    \tikzstyle{hidden neuron2}=[neuron, fill=black!50];
    \tikzstyle{annot} = [text width=4em, text centered]

    \node[input neuron, pin=left:Input : $x_{1}$] (A-1) at (0,-1) {};
    \node[input neuron, pin=left:Input : $\cdots$] (A-2) at (0,-2) {};
    \node[input neuron, pin=left:Input : $x_{n}$] (A-3) at (0,-3) {};
    \node[input neuron ] (B-1) at (3,0) {};
    \node[input neuron ] (B-2) at (3,-1) {};
    \node[input neuron ] (B-3) at (3,-2) {};
    \node[input neuron ] (B-4) at (3,-3) {};
    \path (A-1) edge node[above,sloped]{$\frac{1}{n}$} (B-1);
    \path (A-2) edge node[above,sloped]{$$} (B-1);
    \path (A-3) edge node[above,sloped]{$\frac{1}{n}$} (B-1);
    \path (A-1) edge node[above,sloped]{$+1$} (B-2);
    \path (A-2) edge node[above,sloped]{$+1$} (B-3);
    \path (A-3) edge node[above,sloped]{$+1$} (B-4);
    \node at (3,0.5)[ align=center]{$average$};
    \node at (3,-1.5)[ align=center]{$x_{1}$};
    \node at (3,-2.5)[ align=center]{$\cdots$};
    \node at (3,-3.5)[ align=center]{$x_{n}$};

    \node[input neuron] (C-1) at (6,-1) {};
    \node[input neuron] (C-2) at (6,-2) {};
    \node[input neuron] (C-3) at (6,-3) {};
    \path (B-1) edge node[above,sloped]{$$} (C-1);
    \path (B-1) edge node[above,sloped]{$$} (C-1);
    \path (B-1) edge node[above,sloped]{$$} (C-1);
    \path (B-2) edge node[above,sloped]{$$} (C-1);
    \path (B-3) edge node[above,sloped]{$$} (C-2);
    \path (B-4) edge node[above,sloped]{$$} (C-3);

    \node at (6,-1.5)[ align=center]{$x_{1}-average$};
    \node at (6,-2.5)[ align=center]{$\cdots$};
    \node at (6,-3.5)[ align=center]{$x_{n}-average$};

    \node[input neuron, pin=right:Output : $\mbox{count above mean}$] (D-1) at (9,-2) {};
    \path (C-1) edge node[above,sloped]{$f_{0}(x)$} (D-1);
    \path (C-2) edge node[above,sloped]{$f_{0}(x)$} (D-1);
    \path (C-3) edge node[above,sloped]{$f_{0}(x)$} (D-1);

    \node[input neuron, pin=left:Input : $x_{1}$] (A-1) at (0,-5) {};
    \node[input neuron, pin=left:Input : $\cdots$] (A-2) at (0,-6) {};
    \node[input neuron, pin=left:Input : $x_{n}$] (A-3) at (0,-7) {};
    \node[input neuron ] (B-1) at (3,-4) {};
    \node[input neuron ] (B-2) at (3,-5) {};
    \node[input neuron ] (B-3) at (3,-6) {};
    \node[input neuron ] (B-4) at (3,-7) {};
    \path (A-1) edge node[above,sloped]{$\frac{1}{n}$} (B-1);
    \path (A-2) edge node[above,sloped]{$$} (B-1);
    \path (A-3) edge node[above,sloped]{$\frac{1}{n}$} (B-1);
    \path (A-1) edge node[above,sloped]{$+1$} (B-2);
    \path (A-2) edge node[above,sloped]{$+1$} (B-3);
    \path (A-3) edge node[above,sloped]{$+1$} (B-4);
    \node at (3,-4.5)[ align=center]{$average$};
    \node at (3,-5.5)[ align=center]{$x_{1}$};
    \node at (3,-6.5)[ align=center]{$\cdots$};
    \node at (3,-7.5)[ align=center]{$x_{n}$};

    \node[input neuron] (C-1) at (6,-5) {};
    \node[input neuron] (C-2) at (6,-6) {};
    \node[input neuron] (C-3) at (6,-7) {};
    \path (B-1) edge node[above,sloped]{$$} (C-1);
    \path (B-1) edge node[above,sloped]{$$} (C-1);
    \path (B-1) edge node[above,sloped]{$$} (C-1);
    \path (B-2) edge node[above,sloped]{$$} (C-1);
    \path (B-3) edge node[above,sloped]{$$} (C-2);
    \path (B-4) edge node[above,sloped]{$$} (C-3);

    \node at (6,-5.5)[ align=center]{$average-x_{1}$};
    \node at (6,-6.5)[ align=center]{$\cdots$};
    \node at (6,-7.5)[ align=center]{$average-x_{n}$};

    \node[input neuron, pin=right:Output : $\mbox{count below mean}$] (D-1) at (9,-6) {};
    \path (C-1) edge node[above,sloped]{$g_{0}(x)$} (D-1);
    \path (C-2) edge node[above,sloped]{$g_{0}(x)$} (D-1);
    \path (C-3) edge node[above,sloped]{$g_{0}(x)$} (D-1);

\end{tikzpicture}
}
\caption{Count Features of Time Series}\label{CountFeaturesOfTimeSeries}
\end{center}
\end{figure}
\end{proof}

\renewcommand{\refname}{Reference}

\bibliographystyle{ieeetr}

\bibliography{dnn-bibliography.bib}

\end{document}